 \documentclass[accepted]{uai2023} % after acceptance, for a revised
\usepackage[american]{babel}
\usepackage{natbib} % has a nice set of citation styles and commands
\usepackage{mathtools} % amsmath with fixes and additions
\usepackage{booktabs} % commands to create good-looking tables
\usepackage{tikz} % nice language for creating drawings and diagrams

\usepackage{amsmath}
\usepackage{amssymb}
\usepackage{mathtools}
\usepackage{amsthm}
\usepackage{hhline}
\usepackage{caption}
\usepackage{enumitem}

\usepackage{xcolor,pifont}
\newcommand*\colourcheck[1]{%
  \expandafter\newcommand\csname #1check\endcsname{\textcolor{#1}{\ding{52}}}%
}
\colourcheck{green}

\newcommand*\colourcheckk[1]{%
  \expandafter\newcommand\csname #1times\endcsname{\textcolor{#1}{\ding{54}}}%
}
\colourcheckk{red}

\DeclareMathOperator*{\argmax}{arg\,max}
\DeclareMathOperator*{\argmin}{arg\,min}

\theoremstyle{plain}
\newtheorem{theorem}{Theorem}

\newtheorem{proposition}{Proposition}

\newtheorem{definition}{Definition}

\newtheorem{example}{Example}

\title{Is the Volume of a Credal Set a Good Measure for Epistemic Uncertainty?}

\author[1,3]{\href{mailto:<yusuf.sale@ifi.lmu.de>?Subject=Your UAI 2023 paper}{Yusuf Sale}{}}
\author[2]{Michele Caprio}
\author[1,3]{Eyke Hüllermeier}
\affil[1]{%
Institute of Informatics\\
University of Munich (LMU)\\
Germany
}

\affil[2]{%
    PRECISE Center\\
    Department of Computer and Information Science\\
    University of Pennsylvania\\
    USA
  }

\affil[3]{%
    Munich Center for Machine Learning\\
    Germany

}
  
\begin{document}
\maketitle

\begin{abstract}
Adequate uncertainty representation and quantification have become imperative in various scientific disciplines, especially in machine learning and artificial intelligence. As an alternative to representing uncertainty via one single probability measure, we consider credal sets (convex sets of probability measures). The geometric representation of credal sets as $d$-dimensional polytopes implies a geometric intuition about (epistemic) uncertainty. In this paper, we show that the volume of the geometric representation of a credal set is a meaningful measure of epistemic uncertainty in the case of binary classification, but less so for multi-class classification. Our theoretical findings highlight the crucial role of specifying and employing  uncertainty measures in machine learning in an appropriate way, and for being aware of possible pitfalls. 
\end{abstract}

\section{Introduction}\label{sec:intro}
The notion of uncertainty has recently drawn increasing attention in machine learning (ML) and artificial intelligence (AI) due to the fields' burgeoning relevance for practical applications, many of which have safety requirements, such as in medical domains \citep{lambrou2010reliable, senge_2014_ReliableClassificationLearning, yang2009using} or socio-technical systems \citep{varshney2016engineering,varshney2017safety}. These applications to safety-critical contexts show that a suitable representation and quantification of uncertainty for modern, reliable machine learning systems is imperative.

In general, the literature makes a distinction between \textit{aleatoric} and \textit{epistemic} uncertainties (AU and EU, respectively) \citep{hora1996aleatory}. While the former is caused by the inherent randomness of the data-generating process, EU results from the learner's lack of knowledge regarding the true underlying model; it also includes approximation uncertainty. Since EU can be reduced per se with further information (e.g., via data augmentation using semantic preserving transformations), it is also referred to as reducible uncertainty. 
%This definition is also in line with Ellsberg's notion of ambiguity \citep{ellsberg}.
% Given its definition, \redc{we do not distinguish between the notion of EU and Ellsberg's concept of ambiguity \citep{ellsberg}.} 
In contrast, aleatoric uncertainty, as a property of the data-generating process, is irreducible \citep{hullermeier2021aleatoric}. The importance of distinguishing between different types of uncertainty is reflected in several areas of recent machine learning research, e.g. in Bayesian deep learning \citep{depeweg2018decomposition,kendall2017uncertainties}, in adversarial example detection \citep{smith2018understanding}, or data augmentation in Bayesian classification \citep{kapoor2022uncertainty}. A qualitative representation of total uncertainty, AU, and EU, and of their asymptotic behavior as the number of data points available to the learning agent increases, is given in Figure \ref{fig3}.

\begin{figure}[h!]
\centering
\includegraphics[width=.48\textwidth]{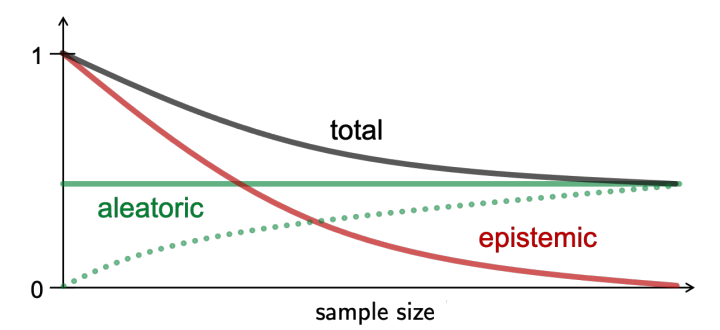}
\caption{Qualitative behavior of total, aleatoric, and epistemic uncertainties depending on the sample size. The dotted line is the difference between total and epistemic uncertainties. This figure replicates \cite[Figure 3]{eyke2}.}
\label{fig3}
\centering
\end{figure}

Typically, uncertainty in machine learning, artificial intelligence, and related fields is expressed solely in terms of probability theory. That is, given a measurable space $(\Omega, \mathcal{A})$, uncertainty is entirely represented by defining one single probability measure $P$ on $(\Omega, \mathcal{A})$. However, representing uncertainty in machine learning is not restricted to classical probability theory; various aspects of uncertainty representation and quantification in ML are discussed by \citet{hullermeier2021aleatoric}. \textit{Credal sets}, i.e., (convex) sets of probability measures, are considered to be very popular models of uncertainty representation, especially in the field of \textit{imprecise probabilities} (IP) \citep{augustin2014introduction,walley}. Credal sets are also very appealing from an ML perspective for representing uncertainty, as they can represent both aleatoric and epistemic uncertainty (as opposed to a single probability measure). Numerous scholars emphasized the utility of representing uncertainty in ML via credal sets, e.g., \textit{credal classification} \citep{zaffalon2002naive, corani2008learning} based on the Imprecise Dirichlet Model (IDM) \citep{walley1996inferences}, generalizing \textit{Bayesian networks} to \textit{credal classifiers} \citep{corani2012bayesian}, or building \textit{credal decision-trees} \citep{abellan2003building}.

Uncertainty representation via credal sets also requires a corresponding \textit{quantification} of the underlying uncertainty, referred to as \textit{credal uncertainty quantification} (CUQ). The task of (credal) uncertainty quantification translates to finding a suitable measure that can accurately reflect the uncertainty inherent to a credal set. In many ML applications, such as active learning \citep{settles2009active} or classification with abstention, there is a need to quantify (predictive) uncertainty in a scalar way. Appropriate measures of uncertainty are often axiomatically justified \citep{bronevich2008axioms, bronevich2010measures}.

\textbf{Contributions.} In this work, we consider the volume of the geometric representation of a credal set on the label space as a quite obvious and intuitively plausible  measure of EU. We argue that this measure is indeed meaningful if we are in a binary classification setting. However, in a multi-class setting, the volume exhibits shortcomings that make it unsuitable for quantifying EU associated with a credal set.

\textbf{Structure of the paper.} The paper is divided as follows. Section \ref{uncertainty} formally introduces the framework we work in, and Section \ref{credal uncertainty quantification} discusses the related literature. Section \ref{geometry of credal uncertainty} presents our main findings, which are further discussed in Section \ref{concl}. Proofs of our theoretical results are given in Appendix \ref{proofs}, and (a version of) Carl-Pajor's theorem, intimately related to Theorem \ref{main_theorem}, is stated in Appendix \ref{hdp}.

%appendix \ref{proofs}

\section{Uncertainty in ML and AI}
\label{uncertainty}
Uncertainty is a crucial concept in many academic and applied disciplines. However, since its definition depends on the specific context a scholar works in, we now introduce the formal framework of supervised learning within which we will examine it.

Let $(\mathcal{X}, \sigma(\mathcal{X}))$, and $(\mathcal{Y}, \sigma(\mathcal{Y}))$ be two measurable spaces, where $\sigma(\mathcal{X})$, and $\sigma(\mathcal{Y})$ are suitable $\sigma$-algebras.
%where $\sigma(\cdot)$ denotes the usual %$\sigma$-operator. 
We will refer to $\mathcal{X}$ as \textit{instance space} (or equivalently, input space) and to $\mathcal{Y}$ as \textit{label space}. Further, the sequence $\{ ({x}_i, y_i )  \}_{i = 1}^n \in (\mathcal{X} \times \mathcal{Y})^n$, is called \textit{training data}. The pairs $(x_i, y_i)$ are realizations of random variables $(X_i, Y_i)$, which are assumed independent and identically distributed (i.i.d.) according to some probability measure $P$ on $(\mathcal{X} \times \mathcal{Y}, \sigma(\mathcal{X} \times \mathcal{Y} ) )$. 

\begin{definition}[Credal set]
\label{definition credal set}
Let $(\Omega, \mathcal{A})$ be a generic measurable space and denote by $\mathcal{M}(\Omega, \mathcal{A})$ the set of all (countably additive) probability measures on $(\Omega, \mathcal{A})$. A convex subset $\mathcal{P} \subseteq \mathcal{M}(\Omega, \mathcal{A})$ is called a \textit{credal set}. 
%If it has finitely many extreme points, it is called a \textit{finitely generated %credal set}.
\end{definition}

Note that in Definition \ref{definition credal set}, the assumption of convexity is quite natural and considered to be rational (see, e.g., \citet{levi2}).
%{\color{blue} MC: here rational is not well defined, and it may incur critiques from %the reviewers. In addition, I'd add a citation on the assumption of convexity being %natural, like a well-known paper that does that.} 
It is also mathematically appealing, since, as shown by \citet[Section 3.3.3]{walley}, the ``lower boundary'' $\underline{P}$ of $\mathcal{P}$, defined as $\underline{P}(A)\coloneqq\inf_{P\in\mathcal{P}}P(A)$, for all $A\in\mathcal{A}$ and called the \textit{lower probability} associated with $\mathcal{P}$, is coherent \citep[Section 2.5]{walley}.

Further, in a supervised learning setting, we assume a \textit{hypothesis space} $\mathcal{H}$, where each hypothesis $h \in \mathcal{H}$ maps a query instance $\mathbf{x_q}$ to a probability measure $P$ on $(\mathcal{Y}, \sigma(\mathcal{Y}))$. We distinguish between different ``degrees'' of uncertainty-aware predictions, which are depicted in Table \ref{table-1}. 
\vspace{-3mm}
\begin{center}
\begin{table}[h]
    \begin{tabular}{ | l | l |  p{2.1cm} |}
    \hline
   \textbf{ Predictor} & \textbf{AU aware?} & \textbf{EU aware?} \\ \hhline{|=|=|=|}
    \small{\vtop{\hbox{\strut Hard label prediction:}\hbox{\strut $h: \mathcal{X} \longrightarrow \mathcal{Y}$}}} & \redtimes & \redtimes  \\ \hline
    \small{\vtop{\hbox{\strut Probabilistic prediction:}\hbox{\strut $h: \mathcal{X} \longrightarrow \mathcal{M}(\mathcal{Y}, \sigma(\mathcal{Y})) $}}} & \greencheck & \redtimes \\ \hline
     \small{\vtop{\hbox{\strut Credal prediction:}\hbox{\strut $h: \mathcal{X} \longrightarrow \text{Cr}(\mathcal{Y})$}}} & \greencheck & \greencheck \\
    \hline
    \end{tabular}
    \caption{Aleatoric uncertainty (AU) and epistemic uncertainty (EU) awareness of different predictors.}
    \label{table-1}
\end{table}
\end{center}
\vspace{-8mm}
We denote by $\text{Cr}(\mathcal{Y})$ the set of all credal sets on $(\mathcal{Y}, \sigma(\mathcal{Y}))$. While probabilistic predictions $h(\mathbf{x_q}) = \hat{y}$ fail to capture the epistemic part of the (predictive) uncertainty, predictions in the form of credal sets $h(\mathbf{x_q}) = \mathcal{P} \subseteq \mathcal{M}(\mathcal{Y}, \sigma(\mathcal{Y}))$ account for both types of uncertainty. It should also be remarked that representing uncertainty is not restricted to the credal set formalism. Another possible framework to represent AU and EU is that of second-order distributions; they are commonly applied in Bayesian learning and  have been recently inspected in the context of uncertainty quantification by \cite{bengs2022pitfalls}. 

In this paper, we restrict our attention to the credal set representation. Given a credal prediction set, it remains to properly quantify the uncertainty encapsulated in it using a suitable measure. Credal set representations are often illustrated in low dimensions (usually $d =2$ or $d=3$). 
%{\color{blue} MC: citation needed here; it has to be a well-known paper that makes %example in a low-dimensional label space.} 
Examples of such geometrical illustrations can be found in the context of machine learning in \citep{hullermeier2021aleatoric} and in imprecise probability theory in \cite[Chapter 4]{walley}. This suggests that a credal set and its geometric representation are strictly intertwined. We will show in the following sections that this intuitive view can have disastrous consequences in higher dimensions and that one should exercise caution in this respect. Furthermore, it remains to be discussed whether a geometric viewpoint on (predictive) uncertainty quantification is in fact sensible. 

\section{Measures of Credal Uncertainty}
\label{credal uncertainty quantification}
In this section we examine some axiomatically defined properties of (credal) uncertainty measures. For a more detailed discussion of various (credal) uncertainty measures in machine learning and a critical analysis thereof, we refer to \cite{hullermeier2022quantification}. 

%This section examines well-received credal uncertainty measures and discusses some %of their desirable properties. 
%This section discusses some desired properties of credal uncertainty measures and briefly examines such well-received measures. 

Let $S$ denote the Shannon entropy \citep{shannon1948mathematical}, whose discrete version is defined as
\begin{align*}
S: \mathcal{M}(\mathcal{Y}, \sigma(\mathcal{Y})) &\rightarrow \mathbb{R}, \\
P &\mapsto  S(P) \coloneqq - \sum_{y \in \mathcal{Y}} P(\{y \}) \log_2 P(\{y\}).
\end{align*}
A suitable measure of credal uncertainty $U:\text{Cr}(\mathcal{Y}) \rightarrow \mathbb{R}$ should satisfy the following axioms proposed by \cite{abellan3,jiro}:
\begin{itemize}
    \item[A1] \textit{Non-negativity} and \textit{boundedness}: 
    \begin{itemize}
        \item[(i)] $U(\mathcal{P}) \geq 0$, for all $\mathcal{P} \in \text{Cr}(\mathcal{Y})$;
        \item[(ii)] there exists $u\in\mathbb{R}$ such that $U(\mathcal{P}) \leq u$, for all $\mathcal{P} \in \text{Cr}(\mathcal{Y})$.
    \end{itemize}
    \item[A2] \textit{Continuity}: $U$ is a continuous functional.
    \item[A3] \textit{Monotonicity}: for all $\mathcal{Q}, \mathcal{P} \in \text{Cr}(\mathcal{Y})$ such that $\mathcal{Q}\subset \mathcal{P}$, we have $U(\mathcal{Q}) \leq U(\mathcal{P})$.
    \item[A4] \textit{Probability consistency}: for all $\mathcal{P} \in \text{Cr}(\mathcal{Y})$ such that $\mathcal{P} = \{P \}$, we have $U(\mathcal{P}) = S(P)$.
    \item[A5] \textit{Sub-additivity}: Suppose $\mathcal{Y} = \mathcal{Y}_1 \times \mathcal{Y}_2$, and let $\mathcal{P}$ be a joint credal set on $\mathcal{Y}$ such that $\mathcal{P}^\prime$ is the marginal credal set on $\mathcal{Y}_1$ and $\mathcal{P}^{\prime\prime}$ is the marginal credal set on $\mathcal{Y}_2$, respectively. Then, we have
        \begin{equation}\label{add_eq}
            U(\mathcal{P}) \leq U(\mathcal{P}^\prime)+ U(\mathcal{P}^{\prime\prime}).
        \end{equation}
    \item[A6] \textit{Additivity}: If $\mathcal{P}^\prime$ and $\mathcal{P}^{\prime\prime}$ are independent, \eqref{add_eq} holds with equality.
\end{itemize}
In axiom A6, independence refers to a suitable notion for independence of credal sets, see e.g.\ \cite{couso}. An axiomatic definition of properties for uncertainty measures is a common approach in the literature \citep{pal1992uncertainty, pal1993uncertainty}. Examples of credal uncertainty measures that satisfy some of the axioms A1--A6 are the maximal entropy \citep{abellan2003maximum} and the generalized Hartley measure \citep{abellan2000non}.

Recall that the lower probability $\underline{P}$ of $\mathcal{P}$ is defined as $\underline{P}(A) \coloneqq \inf_{P \in \mathcal{P}} P(A)$, for all $A\in\sigma(\mathcal{Y})$, and call \textit{upper probability} its conjugate $\overline{P}(A) \coloneqq 1-\underline{P}(A^c)=\sup_{P \in \mathcal{P}} P(A)$, for all $A\in\sigma(\mathcal{Y})$. Since we are concerned with the fundamental question of whether the volume functional is a suitable measure for \textit{epistemic} uncertainty, we replace A4 with the following axiom that better suits our purposes. 

\begin{itemize}[leftmargin=2.4em]
    \item[A4'] \textit{Probability consistency}: $U(\mathcal{P})$ reduces to $0$ as the distance between $\overline{P}(A)$ and $\underline{P}(A)$ goes to $0$, for all $A\in\sigma(\mathcal{Y})$.
\end{itemize}
While A4' addresses solely the epistemic component of uncertainty assoicated with the credal set $\mathcal{P}$, A4 incorporates the aleatoric uncertainty. 
%\redc{Axiom A4 is more general than A4' 
%because it also takes AU into account.} 
Finally, we introduce a seventh axiom that subsumes a desirable property of $U$ proposed by \citet[Theorem 1.A3-A5]{hullermeier2022quantification}. 
\begin{itemize}[leftmargin=2.2em]
    \item[A7] \textit{Invariance}: $U$ is invariant to rotation and translation.
\end{itemize}

Call $d$ the cardinality of the label space $\mathcal{Y}$. In the next section, we will note that many of these axioms are satisfied by the volume operator in the case $d=2$ but can no longer be guaranteed for $d > 2$.

\section{Geometry of Epistemic Uncertainty}
\label{geometry of credal uncertainty}
As pointed out in Section \ref{credal uncertainty quantification}, there is no unambiguous measure of (credal) uncertainty for machine learning purposes. In this section, we present a measure for EU rooted in the geometric concept of volume and show how it is well-suited for a binary classification setting, while it loses its appeal when moving to a multi-class setting. 

Since we are considering a classification setting, we assume that $\mathcal{Y}$ is a finite Polish space so that $|\mathcal{Y} |=d$, for some natural number $d\geq 2$. We also let $\sigma(\mathcal{Y})=2^\mathcal{Y}$ to work with the finest possible $\sigma$-algebra of $\mathcal{Y}$; the results we provide still hold for any coarser $\sigma$-algebra.\footnote{Call $\tau$ the topology on $\mathcal{Y}$. The ideas expressed in this paper can be easily extended to the case where $\mathcal{Y}$ is not Polish. We require it to convey our results without dealing with topological subtleties.} Because $\mathcal{Y}$ is Polish, $\mathcal{M}(\mathcal{Y},\sigma (\mathcal{Y}))$ is Polish as well. In particular, the topology endowed to $\mathcal{M}(\mathcal{Y},\sigma (\mathcal{Y}))$ is the weak topology, which -- because we assumed $\mathcal{Y}$ to be finite -- coincides with the topology $\tau_{\|\cdot\|_2}$ induced by the Euclidean norm. Consider a credal set $\mathcal{P} \subset \mathcal{M}(\mathcal{Y},\sigma (\mathcal{Y}))$, which can be seen as the outcome of a procedure involving an imprecise Bayesian neural network (IBNN) \citep{caprio_IBNN}, or an imprecise neural network (INN) \citep{caprio_INN}; an ensemble-based approach is proposed by \cite{shaker2020aleatoric}.

Since $\mathcal{Y}=\{y_1,\ldots,y_d\}$, each element $P\in\mathcal{P}$ can be seen as a $d$-dimensional probability vector, $P=(p_1,\ldots,p_d)^\top$, where $p_j=P(\{y_j\})$, $j\in\{1,\ldots,d\}$, $p_j\geq 0$, for all $j\in\{1,\ldots,d\}$, and $\sum_{j=1}^d p_j=1$. This entails that if we denote by $\Delta^{d-1}$ the unit simplex in $\mathbb{R}^d$, we have $\mathcal{P}\subset\Delta^{d-1}$, which means that $\mathcal{P}$ is a convex body inscribed in $\Delta^{d-1}$.\footnote{In the remaining part of the paper, we denote by $\mathcal{P}$ both the credal set and its geometric representation, as no confusion arises.}

Intuitively, the ``larger'' $\mathcal{P}$ is, the higher the credal uncertainty. A natural way of capturing the size of $\mathcal{P}$, then, appears to be its volume $\text{Vol}(\mathcal{P})$. Notice that $\text{Vol}(\mathcal{P})$ is a bounded quantity: its value is bounded from below by $0$ and from above by $\sqrt{d}/[(d-1)!]$, the volume of the whole unit simplex $\Delta^{d-1}$. The latter corresponds to the case where $\mathcal{P}=\Delta^{d-1}$, that is, to the case of completely vacuous beliefs: the agent is only able to say that the probability of $A$ is in $[0,1]$, for all $A\in\mathcal{F}$. In this sense, the volume is a measure of the size of set $\mathcal{P}$ that increases the more uncertain the agent is about the elements of $\mathcal{F}$. This argument shows that $\text{Vol}(\mathcal{P})$ is well suited to capture credal uncertainty. But why is it appropriate to describe EU?\footnote{The concept of volume has been explored in the imprecise probabilities literature, see e.g., \cite{bloch}, \cite[Chapter 17]{cuzzo2}, and \cite{teddy_sei}, but, to the best of our knowledge, has never been tied to the notion of epistemic uncertainty. More in general, the geometry of imprecise probabilities has been studied, e.g., by \cite{anel,cuzzo2}.} Think of the extreme case where EU does not exist, so that the agent faces AU only. In that case, they would be able to specify a unique probability measure $P\in\mathcal{M}(\mathcal{Y},\sigma(\mathcal{Y}))$ (or equivalently, $P\in\Delta^{d-1}$), and $\text{Vol}(\{P\})=0$. Hence, if $\text{Vol}(\mathcal{P})>0$, then this means that the agent faces EU. In addition, let $(\mathcal{P}_n)_{n \in \mathbb{N}}$ be a sequence of credal sets on $(\mathcal{Y},\sigma(\mathcal{Y}))$ representing successive refinements of $\mathcal{P}$ computed as new data becomes available to the agent.\footnote{Clearly $|\mathcal{P}_n|=|\mathcal{P}|$, for all $n$.}  
%, and assume $\mathcal{P}_n$ to be connected for all $n$.
%Recall that $\underline{P}_n(A) \coloneqq \inf_{P \in \mathcal{P}_n} P(A)$ is the lower probability associated with $\mathcal{P}_n$, and call \textit{upper probability} its conjugate $\overline{P}_n(A) \coloneqq 1-\underline{P}(A^c)=\sup_{P \in \mathcal{P}_n} P(A)$, $A \in \mathcal{F}$. 
If, after observing enough evidence, the EU is resolved, that is, if $\lim_{n\rightarrow\infty}[\overline{P}_n(A)-\underline{P}_n(A)] = 0$ for all $A\in\mathcal{F}$, 
%\mathcal{P}_n \rightarrow \{P\} as $n\rightarrow\infty$ for some $P\in\mathcal{M}(\Omega,\mathcal{F})$ e.g. in the Hausdorff metric,\footnote{The Hausdorff metric is introduced in equation \eqref{hausd}.} 
we see that the following holds. Sequence $(\mathcal{P}_n )_{n \in \mathbb{N}}$ converges -- say in the Hausdorff metric -- as $n\rightarrow\infty$ to $\mathcal{P}^\star\subset\mathcal{M}(\mathcal{Y},\sigma(\mathcal{Y}))$ such that $|\mathcal{P}^\star|=|\mathcal{P}_n|$, for all $n$, and all the elements of $\mathcal{P}^\star$ are equal to $P^\star$, the (unique) probability measure that encapsulates the AU.\footnote{Technically $\mathcal{P}^\star$ is a multiset, that is, a set where multiple instances for each of its elements are allowed.} Through the learning process, we refine our estimates for the ``true'' underlying aleatoric uncertainty (pertaining to $P^\star$), which is left after all the EU is resolved. Then, the geometric representation of $\mathcal{P}^\star$ is a point whose volume is $0$. Hence, we have that the volume of $\mathcal{P}_n$ converges from above to $0$ (that is, it possesses the continuity property), which is exactly the behavior we would expect as EU resolves.

As we shall see, while this intuitive explanation holds if $d=2$, for $d>2$, continuity fails, thus making the volume not suited to capture EU in a multi-class classification setting. We also show in Theorem \ref{main_theorem} that the volume lacks robustness in higher dimensions. Small perturbations to the boundary of a credal set make its volume vary significantly. This may seriously hamper the results of a study, leading to potentially catastrophic consequences in downstream tasks.

\subsection{$\text{Vol}(\mathcal{P})$: a good measure for EU, but only if $d=2$}\label{continuity}
Let $d=2$ so that $\mathcal{P}$ is a subset of $\Delta^1$, the segment linking the points $(1,0)$ and $(0,1)$ in a $2$-dimensional Cartesian plane. Notice that in this case, the volume $\text{Vol}(\mathcal{P})$ corresponds to the length of the segment. In this context, $\text{Vol}(\mathcal{P})$ is an appealing measure to describe the EU associated with the credal set $\mathcal{P}$.
\begin{proposition}\label{prop-1}
    $\text{Vol}(\cdot)$ satisfies axioms A1--A3, A4', A5 and A7 of Section \ref{credal uncertainty quantification}.
\end{proposition}
Let us now discuss additivity (axiom A6 of Section \ref{credal uncertainty quantification}). Suppose the label space $\mathcal{Y}=\{(y_1,y_2),(y_3,y_4)\}$ can be written as $\mathcal{Y}_1 \times \mathcal{Y}_2$, where $\mathcal{Y}_1=\{y_1,y_3\}$ and $\mathcal{Y}_2=\{y_2,y_4\}$. Let $\mathcal{P}$ be a joint credal set on $\mathcal{Y}$ such that $\mathcal{P}^\prime$ is the marginal credal set on $\mathcal{Y}_1$ and $\mathcal{P}^{\prime\prime}$ is the marginal credal set on $\mathcal{Y}_2$. In the proof of Proposition \ref{prop-1}, we show that if $y_1\neq y_3$ and $y_2\neq y_4$,\footnote{This implies that $|\mathcal{Y}|=|\mathcal{Y}_1|=|\mathcal{Y}_2|=2$.} then the volume is sub-additive. Suppose instead now that $y_1=y_3=y_\star$, so that $|\mathcal{Y}|=|\mathcal{Y}_2|=2$ and $|\mathcal{Y}_1|=1$.\footnote{A similar argument will hold if we assume $y_2=y_4=y^\star$, so that $|\mathcal{Y}|=|\mathcal{Y}_1|=2$ and $|\mathcal{Y}_2|=1$.}  Then, the marginal $\text{marg}_{\mathcal{Y}_1}(P)=P^\prime$ of any $P\in\mathcal{P}$ on $\mathcal{Y}_1$ will give probability $1$ to $y_\star$; in formulas, $P^\prime(y_\star)=1$. This entails that $\mathcal{P}^\prime=\{P^\prime\}$ is a singleton and that its geometric representation is a point.\footnote{Or, alternatively, $\mathcal{P}^\prime$ is a multiset whose elements are all equal.} Then, for all $P\in\mathcal{P}$, $P((y_1,y_2))=P^{\prime\prime}(y_2)$ and $P((y_3,y_4))=P^{\prime\prime}(y_4)$, where $\text{marg}_{\mathcal{Y}_2}(P)=P^{\prime\prime}$ is the marginal of any $P \in\mathcal{P}$ on $\mathcal{Y}_2$. 

In turn, this line of reasoning implies that $ \text{Vol}(\mathcal{P}^\prime) + \text{Vol}(\mathcal{P}^{\prime\prime}) = 0 + \text{Vol}(\mathcal{P}) = \text{Vol}(\mathcal{P})$, which shows that the volume is additive in this case. 

This situation corresponds to an instance of strong independence (SI) \citep[Section 3.5]{couso}. We have SI if and only if
\begin{align}\label{strong}
    \begin{split}
        \mathcal{P}=\text{Conv}(\{P \in \mathcal{M}(\mathcal{Y},\sigma(\mathcal{Y})) \text{: } &\text{marg}_{\mathcal{Y}_1}(P) \in \mathcal{P}^\prime\\
        \text{and } &\text{marg}_{\mathcal{Y}_2}(P) \in \mathcal{P}^{\prime\prime}\}).
    \end{split}
\end{align}
%where $\text{marg}_{\mathcal{Y}_1}(P)$ denotes the marginal of $P$ on $\mathcal{Y}_1$, and similarly for $\text{marg}_{\mathcal{Y}_2}(P)$. 
In other words, there is complete lack of interaction between the probability measure on $\mathcal{Y}_1$ and those on $\mathcal{Y}_2$. To see that this is the case, recall that $\mathcal{P}$ is a credal set, and so is convex; recall also that $\mathcal{P}^\prime=\{P^\prime\}$ is a singleton. Then, pick any $P^{ex}\in\text{ex}(\mathcal{P})$, where $\text{ex}(\mathcal{P})$ denotes the set of extreme elements of $\mathcal{P}$. We have that $\text{marg}_{\mathcal{Y}_1}(P^{ex})=P^\prime$, and so $\text{marg}_{\mathcal{Y}_2}(P^{ex})\in \text{ex}(\mathcal{P}^{\prime\prime})$. With a slight abuse of notation, we can write $\text{ex}(\mathcal{P})=\{P^\prime\} \times \text{ex}(\mathcal{P}^{\prime\prime})$.  This immediately implies that the equality in (\ref{strong}) holds. As pointed out in \cite[Section 3.5]{couso}, SI implies independence of the marginal sets, epistemic independence of the marginal experiments, and independence in the selection \citep[Sections 3.1, 3.4, and 3.5, respectively]{couso}. It is, therefore, a rather strong notion of independence.

The volume is also trivially additive if $(y_1,y_2)=(y_3,y_4)$, but in that case $\mathcal{Y}$ would be a multiset. 

The argument put forward so far can be summarized in the following proposition.
\begin{proposition}\label{prop-2}
    Let $\mathcal{Y}=\{(y_1,y_2),(y_3,y_4)\}$. $\text{Vol}(\cdot)$ satisfies axiom A6 if we assume the instance of SI given by either of the following
    \begin{itemize}
        \item $y_1=y_3$,
        \item $y_2=y_4$,
        \item $y_1=y_3$ and $y_2=y_4$.
    \end{itemize}
\end{proposition}

If $d>2$, the volume ceases to be an appealing measure for EU. This is because quantifying the uncertainty associated with a credal set becomes challenging due to the dependency of the volume on the dimension.
% \redc{To see why, notice that the volume depends on the dimension of the space we consider.} {\color{blue} MC: maybe ``This because the volume depends on the dimension of the space we consider.''} 
So far, we have written $\text{Vol}$ in place of $\text{Vol}_{d-1}$ to ease notation, but for $d>2$ the dimension with respect to which the volume is taken becomes crucial. Let us give a simple example to illustrate this.

\begin{example}\label{failure}
    Let $d=3$, so that the unit simplex is $\Delta^{3-1}=\Delta^2$, the triangle whose extreme points are $(1,0,0)$, $(0,1,0)$, and $(0,0,1)$ in a $3$-dimensional Cartesian plane (the purple triangle in Figure \ref{fig4}). Consider a sequence $(\mathcal{P}_n)$ of credal sets whose geometric representations are triangles, and suppose their height reduces to $0$ as $n\rightarrow\infty$, so that the (geometric representation of) $\mathcal{P}_\infty$ -- the limit of $(\mathcal{P}_n)$ in the Hausdorff metric -- is a segment. The limiting set $\mathcal{P}_\infty$, then, is not of full dimensionality that is, its geometric representation is a proper subset of $\Delta^1$, while the geometric representation of $\mathcal{P}_n$ is a proper subset of $\Delta^2$, for all $n$. This implies that $\text{Vol}_2(\mathcal{P}_\infty)=0$, but -- unless $\mathcal{P}_\infty$ is a degenerate segment, i.e. a point -- $\text{Vol}_1(\mathcal{P}_\infty)>0$. As we can see, the EU has not resolved, yet $\mathcal{P}_\infty$ has a zero $2$-dimensional volume; this is clearly undesirable. It is easy to see how this problem exacerbates in higher dimensions.
    %\demo
\end{example}

There are two possible ways one could try to circumvent the issue in Example \ref{failure}; alas, both exhibit shortcomings, that is, at least one of the axioms A1--A3, A4', A5--A7 in Section \ref{credal uncertainty quantification} is not satisfied.
%for both of them axiom (A3) fails to hold. 
The first one is to consider the volume operator $\text{Vol}(\mathcal{P})$ as the volume taken with respect to the space in which set $\mathcal{P}$ is of full dimensionality. In this case, we immediately see how A2 fails. Considering again the sequence in Example \ref{failure}, we would have a sequence $(\mathcal{P}_n)$ whose volume $\text{Vol}_2(\mathcal{P}_n)$ is going to zero. However, in the limit, its volume $\text{Vol}_1(\mathcal{P}_\infty)$ would be positive. Axiom A3 fails as well: consider a credal set $\mathcal{P}$ whose representation is a triangle having base $b$ and height $h$ and suppose $h<2$. Consider then a credal set $\mathcal{Q} \subsetneq \mathcal{P}$ whose representation is a segment having length $\ell=b$. Then, $\text{Vol}_2(\mathcal{P})=b\cdot h/2 < b$, while $\text{Vol}_1(\mathcal{Q})=\ell=b$.

The second one is to consider lift probability sets; let us discuss this idea in depth. Let  $d,d^\prime\in\mathbb{N}$, and let $d^\prime < d$. Call 
$$O(d^\prime,d)\coloneqq\{V \in \mathbb{R}^{d^\prime\times d} : VV^\top=I_{d^\prime}\},$$
where $I_{d^\prime}$ is the $d^\prime$-dimensional identity matrix. That is, $O(d^\prime,d)$ is the \textit{Stiefel manifold} of $d^\prime\times d$ matrices with orthonormal rows \citep{lim2}. Then, for any $V\in O(d^\prime,d)$ and any $b\in \mathbb{R}^{d^\prime}$, define
$$\varphi_{V,b}: \mathbb{R}^d \rightarrow \mathbb{R}^{d^\prime}, \quad x \mapsto \varphi_{V,b}(x)\coloneqq Vx + b.$$
Suppose now that, for some $n$, (the geometric representation of) $\mathcal{P}_n$ is a proper subset of $\Delta^{d-1}$, while (the geometric representation of) $\mathcal{P}_{n+1}$ is a proper subset of $\Delta^{d^{\prime}-1}$. Pick any $V\in O(d^\prime,d)$ and any $b\in\mathbb{R}^{d^\prime}$; an embedding of $\mathcal{P}_{n+1}$ in $\Delta^{d-1}$ is a set $K$ such that for all $x\in K$, there exists a probability vector $p\in\mathcal{P}_{n+1}$ such that $\varphi_{V,b}(x)=p$. 
%\begin{itemize}
    %\item compact (this assumption is needed because we need compactness for our results in the paper);
    %\item for all $x\in K$, there exists probability vector $p\in\mathcal{P}_{n+1}$ such that $\varphi_{V,b}(x)=p$.
%\end{itemize}
%closed and 
Call $\Phi^+(\mathcal{P}_{n+1},d)$ the set of embeddings of $\mathcal{P}_{n+1}$ in $\Delta^{d-1}$, and assume that it is nonempty. 
%\footnote{This should be easily %verifiable.}

%{\color{blue} Maybe we need to require $\Phi^+(\mathcal{P}_{n+1},d)$ to be closed as well; if so, it has to be closed with respect to the topology induced by the Hausdorff metric in equation (3) in the write-up. I know that compact sets in the Hausdorff space are closed, but $\Phi^+(\mathcal{P}_{n+1},d)$ is a set of sets, so maybe we should require it?}
Then, define
$$\breve{\mathcal{P}}_{n+1}\coloneqq\argmin_{K\in \Phi^+(\mathcal{P}_{n+1},d)} \left| \text{Vol}_{d-1} (K) - \text{Vol}_{d^\prime-1} (\mathcal{P}_{n+1}) \right|;$$
we call it the \textit{lift probability set} for the heuristic similarity with lift zonoids \citep{mosler}. We define it in this way because we want the $d$-dimensional set whose (full dimensionality) volume is the closest possible to the ($d^\prime$-dimensional) volume of $\mathcal{P}_{n+1}$. A simple example is the following. Suppose the geometric representation of $\mathcal{P}_n$ is a proper subset of $\Delta^2$, and that the geometric representation of $\mathcal{P}_{n+1}$ is a proper subset of $\Delta^1$. So the former is a subset of $\mathbb{R}^2$, and the latter is a segment in $\mathbb{R}$. Then, a possible $\breve{\mathcal{P}}_{n+1}$ is any triangle in $\Delta^2$ whose height $h$ is $2$ and whose base length $b$ is equal to the length $\ell$ of the segment representing $\mathcal{P}_{n+1}$. This because the area of such $\breve{\mathcal{P}}_{n+1}$ is $b \cdot h/2$; if $h=2$ and $b=\ell$, then $\text{Vol}_2(\breve{\mathcal{P}}_{n+1})=\text{Vol}_1(\mathcal{P}_{n+1})$, which is what we wanted. A visual representation is given in Figure \ref{fig4}.

\begin{figure}[h!]
\centering
\includegraphics[width=.39\textwidth]{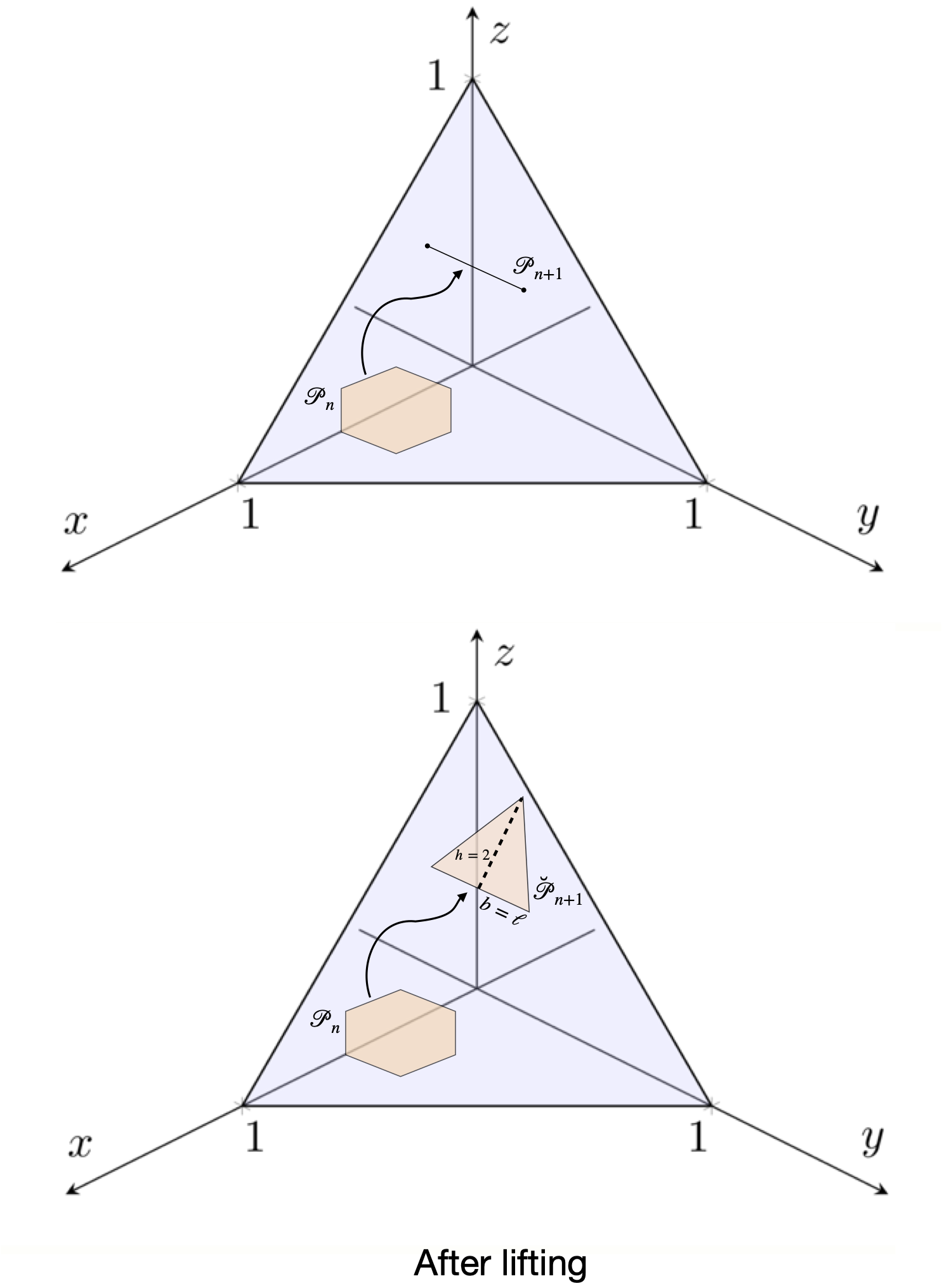}
\caption{A visual representation of a lift probability set.}
\label{fig4}
\centering
\end{figure}

Notice that $\breve{\mathcal{P}}_{n+1}$ is well defined because $\Phi^+(\mathcal{P}_{n+1},d) \subset 2^{\Delta^{d-1}}$, and $\Delta^{d-1}$ is compact.\footnote{If the $\argmin$ is not a singleton, pick any of its elements.} We can then compare $\text{Vol}_{d-1}(\mathcal{P}_n)$ and of $\text{Vol}_{d-1}(\breve{\mathcal{P}}_{n+1})$, and also compute the relative quantity
$$\frac{\left|\text{Vol}_{d-1}(\mathcal{P}_n) - \text{Vol}_{d-1}(\breve{\mathcal{P}}_{n+1})\right|}{\text{Vol}_{d-1}(\mathcal{P}_n)}$$
that captures the variation in volume between $\mathcal{P}_n$ and $\breve{\mathcal{P}}_{n+1}$. Alas, in this case, too, it is easy to see how A2 fails. Consider the same sequence as in Example \ref{failure}. We would have that $\text{Vol}_2(\mathcal{P}_n)$ goes to zero as $n\rightarrow\infty$, but $\text{Vol}_2(\breve{\mathcal{P}}_\infty)>0$. Axiom A3 may fail as well since we could find credal sets $\mathcal{P} \subset \Delta^{d-1}$ and $\mathcal{Q} \subset \Delta^{d^\prime-1}$ such that $\mathcal{Q} \subsetneq \mathcal{P}$, but $\breve{\mathcal{Q}} \not\subset \mathcal{P}$.

\subsection{Lack of robustness in higher dimensions}\label{instab}
In this section, we show how, if we measure the EU associated with a credal set on the label space using the volume, as the number of labels grows, ``small'' changes of the uncertainty representation may lead to catastrophic consequences in downstream tasks. 

 For a generic compact set $K\in\mathbb{R}^d$ and a positive real $r$, the \textit{$r$-packing} of $K$, denoted by $\text{Pack}_r(K)$, is the collection of sets $K^\prime$ that satisfy the following properties
\begin{itemize}
    \item[(i)] $K^\prime \subset K$,
    \item[(ii)] $\cup_{x\in K^\prime}B_r^d(x) \subset K$, where $B_r^d(x)$ denotes the ball of radius $r$ in space $\mathbb{R}^d$ centered at $x$,
    \item[(iii)] the elements of $\{B_r^d(x)\}_{x\in K^\prime}$ are pairwise disjoint,
    \item[(iv)] there does not exist $x^\prime \in K$ such that (i)-(iii) are satisfied by $K^\prime\cup\{x^\prime\}$.
\end{itemize}
The \textit{packing number} of $K$, denoted by $N^{\text{pack}}_r(K)$, is given by $\max_{K^\prime\in{\text{Pack}_r(K)}}|K^\prime|$. We also let $K^\star_r\coloneqq\argmax_{K^\prime\in{\text{Pack}_r(K)}}|K^\prime|$ and $\tilde{K}_r\coloneqq\cup_{x\in K^\star_r} B^d_r(x)$. Notice that 
\begin{equation}\label{volume_eq}
    \text{Vol}(\tilde{K}_r)=c(r,d,K)\text{Vol}(K),
\end{equation}
where 
\begin{equation}\label{decr_c}
\begin{split}
   c(r,d,K) &\in (0,1] \text{, for all } r>0,\\
   \text{and} \quad c(r,d,K) &\leq c(r-\epsilon,d,\check{K}) \text{, for all } \epsilon>0,
\end{split}
\end{equation}
where $\check{K}$ is any compact set in $\mathbb{R}^d$, possibly different than $K$. That is, we can always find a real number $c(r,d,K)$ depending on the dimension $d$ of the Euclidean space, on the radius $r$ of the balls, and on the set $K$ of interest, that relates the volume of $K$ and that of $\tilde{K}_r$. Being in $(0,1]$, it takes into account the fact that since $\tilde{K}_r$ is a union of pairwise disjoint balls \textit{within} $K$, its volume cannot exceed that of $K$. This is easy to see in Figure \ref{fig2}. The second condition in \eqref{decr_c} states that irrespective of the compact set of interest, we retain more of the volume of the original set if we pack it using balls of a smaller radius. 

To give a simple illustration, consider $r_1,r_2>0$ such that $r_1 \leq r_2$. Then, by \eqref{volume_eq} and \eqref{decr_c}, we have that $\text{Vol}({K})-\text{Vol}(\tilde{K}_{r_2}) =  \text{Vol}({K})[1-c(r_2,d,K)] \geq \text{Vol}({K})[1-c(r_1,d,K)] = \text{Vol}({K}) - \text{Vol}(\tilde{K}_{r_1})$. This means that the difference in volume between $K$ and $\tilde{K}_{r_2}$ is larger than that between $K$ and $\tilde{K}_{r_1}$.

Let $\mathcal{K}(\mathbb{R}^d)$ be the class of compact sets in $\mathbb{R}^d$, and call $c(r,d)\coloneqq\max_{K\in\mathcal{K}(\mathbb{R}^d)} c(r,d,K)$. As $r$ goes to $0$, $c(r,d)$ increases to its optimal value that we denote as $c^\star(d)$. The values of $c^\star(d)$ have only been found for $d\in\{1,2,3,8,24\}$ \citep{cohn,maryna}. The fact that $c(r,d)$ increases as $r$ decreases to $0$ captures the idea that using balls of smaller radius leads to a better approximation of the volume of the compact set $K$ in $\mathbb{R}^d$ that is being packed.

\begin{figure}[h!]
\centering
\includegraphics[width=.3\textwidth]{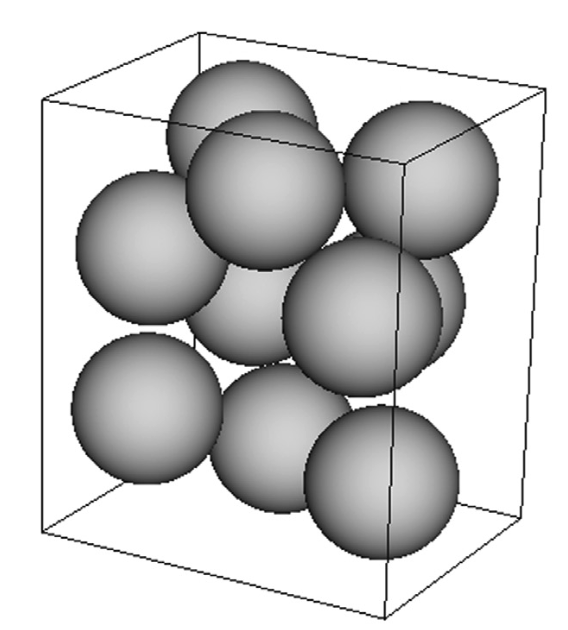}
\caption{A representation of $\tilde{K}_r$, for some $r>0$, where $K$ is a parallelepiped in $\mathbb{R}^3$. This figure replicates \cite[Figure 4]{hifi}.}
\label{fig2}
\centering
\end{figure}

Suppose our credal set $\mathcal{P}$ is compact, so to be able to use the concepts of $r$-packing and packing number. Consider then a set $\mathcal{Q} \subset \mathcal{M}(\Omega,\mathcal{F})$ that satisfies the following three properties:
\begin{enumerate}
    \item[(a)] $\mathcal{Q}\subsetneq \mathcal{P}$, so that $\mathcal{Q}^\prime\coloneqq\mathcal{P}\setminus\mathcal{Q} \neq \emptyset$,
    \item[(b)] $d_H(\mathcal{P},\mathcal{Q})= \epsilon$, for some $\epsilon>0$,
    \item[(c)] $\epsilon$ is such that we can find $r>0$ for which $N^{\text{pack}}_r(\mathcal{P}) \geq N^{\text{pack}}_{r-\epsilon}(\mathcal{Q}^\prime)$.
    %\footnote{In light of property (c), $\epsilon$ must be in $(0,r)$.}
\end{enumerate}
Property (a) tells us that $\mathcal{Q}$ is a proper subset of $\mathcal{P}$. Let $d_2$ denote the metric induced by the Euclidean norm $\|\cdot \|_2$. Property (b) tells us that the Hausdorff distance 
\begin{equation}\label{hausd}
    d_H(\mathcal{P},\mathcal{Q})=\max\left\lbrace{ \max_{P\in\mathcal{P}}d_2(P,\mathcal{Q}) ,  \max_{Q\in\mathcal{Q}}d_2(\mathcal{P},Q)}\right\rbrace
\end{equation}
between $\mathcal{P}$ and $\mathcal{Q}$ is equal to some $\epsilon>0$. Property (c) ensures that $\epsilon$ is ``not too large''. To understand why, notice that if $\epsilon$ is ``large'', that is, if it is close to $r$, then the packing number of $\mathcal{Q}^\prime \subsetneq \mathcal{P}$ using balls of radius $r-\epsilon$ can be larger than the packing number of $\mathcal{P}$ using balls of radius $r$.\footnote{Because $\mathcal{Q}\subsetneq \mathcal{P}$ and $d_H(\mathcal{P},\mathcal{Q})= \epsilon$, packing using balls of radius $r-\epsilon$ is a sensible choice.} Requiring (c) ensures us that this does not happen, and therefore that $\epsilon$ is ``small''. A representation of $\mathcal{P}$ and $\mathcal{Q}$ satisfying (a)--(c) is given in Figure \ref{fig1}. A (possibly very small) change in uncertainty representation is captured by a situation in which the agent specifies credal set $\mathcal{Q}$ in place of $\mathcal{P}$. We are ready to state the main result of this section.

\begin{figure}[h!]
\centering
\includegraphics[width=.4\textwidth]{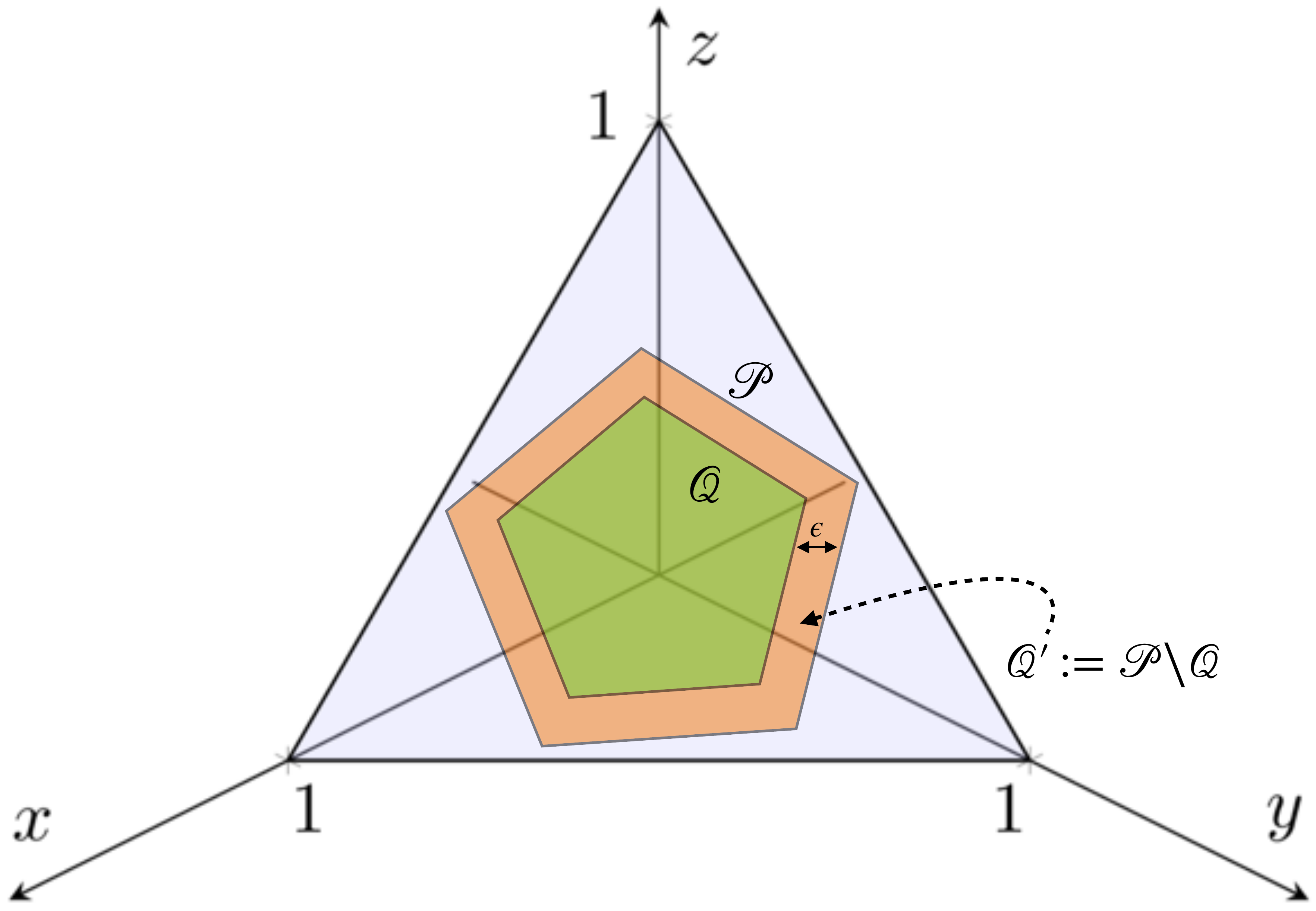}
\caption{A representation of $\mathcal{P}$ (the orange pentagon) and $\mathcal{Q}$ (the green pentagon) satisfying (a)-(c) when the dimension of state space $\Omega$ is $d=3$. The unit simplex $\Delta^2$ in $\mathbb{R}^3$ is given by the purple triangle whose vertices are the elements of the basis of $\mathbb{R}^3$, i.e., $e_1=(1,0,0)$, $e_2=(0,1,0)$, and $e_2=(0,0,1)$.}
\label{fig1}
\centering
\end{figure}

\begin{theorem}\label{main_theorem}
   Let $\Omega$ be a finite Polish space so that $|\Omega|=d$, and let $\mathcal{F}=2^\Omega$. Pick any compact set $\mathcal{P}\subset\mathcal{M}(\Omega,\mathcal{F})$, and any set $\mathcal{Q}$ that satisfies (a)-(c). The following holds
   \begin{equation}\label{main_eq}
       \frac{\text{Vol}(\mathcal{P})-\text{Vol}(\mathcal{Q}^\prime)}{\text{Vol}(\mathcal{P})} \geq 1- \left( 1-\frac{\epsilon}{r} \right)^d.
   \end{equation}
\end{theorem}

Notice that we implicitly assumed that at least a $\mathcal{Q}$ satisfying (a)-(c) exists. We have that $[\text{Vol}(\mathcal{P})-\text{Vol}(\mathcal{Q}^\prime)]/\text{Vol}(\mathcal{P}) \in [0,1]$; in light of this, since $1-(1-\epsilon/r)^d \rightarrow 1$ as $d\rightarrow \infty$, Theorem \ref{main_theorem} states that as $d$ grows, most of the volume of $\mathcal{P}$ concentrates near its boundary. 

As a result, if we use the volume operator as a metric for the EU, this latter is very sensitive to perturbations of the boundary of the (geometric representation of the) credal set; this is problematic for credal sets in the context of ML. Suppose we are in a multi-classification setting such that the cardinality of $\mathcal{Y}$ is some large number $d$. Suppose that two different procedures produce two different credal sets on $\mathcal{Y}$; call one $\mathcal{P}$ and the other $\mathcal{Q}$, and suppose $\mathcal{Q}$ satisfies (a)-(c). This means that the uncertainty representations associated with the two procedures differ only by a ``small amount''. For instance, this could be the result of an agent specifying ``slightly different'' credal prior sets. This may well happen since defining the boundaries of credal sets is usually quite an arbitrary task to perform. Then, this would result in a (possibly massive) underestimation of the epistemic uncertainty in the results of the analysis, which would potentially translate in catastrophic consequence in downstream tasks. In Example \ref{example-2}, we describe a situation in which Theorem \ref{main_theorem} is applied to credal prior sets.

\begin{example}\label{example-2}
 Assume for simplicity that the parameter space $\Theta$ is finite and that its cardinality is $\mathfrak{c}$. Suppose an agent faces complete ignorance regarding the probabilities to assign to the elements of $2^\Theta$. Although tempting, there is a pitfall in choosing the whole simplex $\Delta^{\mathfrak{c}-1}$ as the credal prior set. As shown by \citet[Chapter 5]{walley}, completely vacuous beliefs -- captured by choice of $\Delta^{\mathfrak{c}-1}$ as a credal prior set -- cannot be Bayes-updated. This means that the posterior credal set will again be $\Delta^{\mathfrak{c}-1}$: no large amount of data is enough to swamp the prior. Instead, suppose that the agent considers a credal prior set  $\Delta^{\mathfrak{c}-1}_{\epsilon}$ that satisfies (a)--(c). If $\mathfrak{c}$ is large enough, this would mean that $\text{Vol}(\Delta^{\mathfrak{c}-1}_{\epsilon})$ is much smaller than $\text{Vol}(\Delta^{\mathfrak{c}-1})$.
 \end{example}

\iffalse 
This entails the following. Suppose that in a procedure involving a finite label space $\mathcal{Y}$ -- whose cardinality is some large number $d$ -- we slightly underestimate the ambiguity we face. This may happen if the credal set elicited at the beginning of the analysis is (ever so slightly) misspecified, that is, we select a set $\mathcal{Q}$ that satisfies (a)-(c) in lieu of a ``larger'' $\mathcal{P}$. Often times this will be the case: agents do not realize they face ambiguity, as observed in \cite{berger} and in the de Finetti lecture delivered by Berger at ISBA 2021;  
%There, Berger points out how most people tend to under-report variance; the folklore says by a factor of $3$. 
people simply think that they know more than they actually do. Then, this would results in a (possibly massive) underestimation of the epistemic uncertainty in the results of the analysis, which would potentially translate in catastrophic consequence in downstream tasks.
\fi 

Two remarks are in order. First, in the binary classification setting (that is, when $d=2$), the lack of robustness of the volume highlighted by Theorem \ref{main_theorem} is not an issue since $1-(1-\epsilon/r)^d$ is approximately $1$ only when the cardinality $|\mathcal{Y}|=d$ is large. Second, Theorem \ref{main_theorem} is intimately related to Carl-Pajor's Theorem \citep[Theorem 1]{pajor}; this implies that in the future, more techniques from high-dimensional geometry may become useful in the study of epistemic, and potentially also aleatoric, uncertainties.\footnote{We state (a version of) Carl-Pajor's Theorem in Appendix \ref{hdp}.}

\section{Conclusion}\label{concl}
Credal sets provide a flexible and powerful formalism for representing uncertainty in various scientific disciplines. In particular, uncertainty representation via credal sets can capture different degrees of uncertainty and allow for a more nuanced representation of epistemic and aleatoric uncertainty in machine learning systems. Moreover, the corresponding geometric representation of credal sets as $d$-dimensional polytopes enables a thoroughly intuitive view of uncertainty representation and quantification.  

In this paper, we showed that the volume of a credal set is a sensible measure of epistemic uncertainty in the context of binary classification, as it enjoys many desirable properties suggested in the existing literature. On the other side, the volume forfeits these properties in a multi-class classification setting, despite its intuitive meaningfulness. 

In addition, this work stimulates a fundamental question as to what extent a geometric approach to uncertainty quantification (in ML) is sensible.

This is the first step toward studying the geometric properties of (epistemic) uncertainty in AI and ML. In the future, we plan to explore the geometry of aleatoric uncertainty and introduce techniques from high-dimensional geometry and high-dimensional probability to enhance and deepen the study of EU and AU in the contexts of AI and ML.

\iffalse 
This is just meant as a side note, and will potentially not appear in the final version of the paper:

\begin{theorem}
Let $B_{1,d}$ denote the $d$-dimensional unit euclidean ball, and let $\mathcal{P} \subset B_{1,d}$ be a polytope with $m \in \mathbb{N}$ vertices. Then, we have
\begin{align}
  \frac{\text{Vol}(\mathcal{P})}{\text{Vol}(B_{1,d})} \leq \left(4 \sqrt{\frac{\log m}{d}}   \right)^d. 
\end{align}
\end{theorem}
\fi

\begin{contributions} % will be removed in pdf for initial submission 
					  % (without ‘accepted’ option in \documentclass)
                      % so you can already fill it to test with the
                      % ‘accepted’ class option
Yusuf Sale and Michele Caprio contributed equally to this paper.
\end{contributions}

\begin{acknowledgements} % will be removed in pdf for initial submission,
						 % (without ‘accepted’ option in \documentclass)
                         % so you can already fill it to test with the
                         % ‘accepted’ class option
Michele Caprio would like to acknowledge partial funding by the Army Research Office (ARO MURI W911NF2010080). Yusuf Sale is supported by the DAAD programme Konrad Zuse Schools of Excellence in Artificial Intelligence, sponsored by the Federal Ministry of Education and Research.
\end{acknowledgements}

% References
\bibliography{uai2023-template}

\newpage
\onecolumn
\appendix

\appendix
\section{Proofs}\label{proofs}
\begin{proof}[Proof of Proposition \ref{prop-1}]
Let $\mathcal{P},\mathcal{Q} \subset \Delta(\mathcal{Y},\sigma(\mathcal{Y}))$ be credal sets, and assume $|\mathcal{Y}|=2$. Then we have the following.
    \begin{itemize}
        \item $\text{Vol}(\mathcal{P}) \geq 0$ and $\text{Vol}(\mathcal{P}) \leq \text{Vol}(\Delta^{2-1})=\sqrt{2}$. Hence $\text{Vol}(\cdot)$ satisfies A1.
        \item The volume being a continuous functional is a well-known fact that comes from the continuity of the Lebesgue measure, so $\text{Vol}(\cdot)$ satisfies A2.
        \item $\mathcal{Q}\subset \mathcal{P} \implies \text{Vol}(\mathcal{Q}) \leq \text{Vol}(\mathcal{P})$. This comes from the fundamental property of the Lebesgue measure, so $\text{Vol}(\cdot)$ satisfies A3.
        \item Consider a sequence $(\mathcal{P}_n)$ of credal sets on $(\mathcal{Y},\sigma(\mathcal{Y}))$ such that $\lim_{n\rightarrow \infty} [\overline{P}_n(A)-\underline{P}_n(A)]=0$, for all $A\in\sigma(\mathcal{Y})$. Then, this means that there exists $N\in\mathbb{N}$ such that for all $n\geq N$, the geometric representation of $\mathcal{P}_n$ is a subset of the geometric representation of $\mathcal{P}_{n+1}$. In addition, the limiting element of $(\mathcal{P}_n)$ is a (multi)set $\mathcal{P}^\star$ whose elements are all equal to $P^\star$, so its geometric representation is a point and its volume is $0$. Hence, probability consistency is implied by continuity A3, so $\text{Vol}(\cdot)$ satisfies A4'.
        \item The volume is invariant to rotation and translation. This is a well-known fact that comes from the fundamental property of the Lebesgue measure, so $\text{Vol}(\cdot)$ satisfies A7.
    \end{itemize}
Let us now show that the volume operator satisfies sub-additivity A5. Let $\mathcal{Y}=\mathcal{Y}_1\times\mathcal{Y}_2$. In addition, suppose we are in the general case in which $|\mathcal{Y}|=|\mathcal{Y}_1|=|\mathcal{Y}_2|=2$. In particular, let $\mathcal{Y}=\{(y_{1},y_{2}),(y_{3},y_{4})\}$, so that $\mathcal{Y}_1=\{y_{1},y_{3}\}$ and $\mathcal{Y}_2=\{y_{2},y_{4}\}$. Suppose also $y_1\neq y_3$ and $y_2\neq y_4$. Now, pick any probability measure $P$ on $\mathcal{Y}$. In general, we would have that its marginal $\text{marg}_{\mathcal{Y}_1}(P)=P^\prime$ on $\mathcal{Y}_1$ is such that  $P^\prime(y_i)=\sum_j P((y_i,y_j))$. Similarly for marginal $\text{marg}_{\mathcal{Y}_2}(P)=P^{\prime\prime}$ on $\mathcal{Y}_2$. In our case, though, the computation is easier. To see this, fix $y_1$. Then, we should sum over $j$ the probability of $(y_1,y_j)$, $y_j\in\mathcal{Y}_2$. But the only pair $(y_1,y_j)$ is $(y_1,y_2)$. A similar argument holds if we fix $y_3$, or any of the elements of $\mathcal{Y}_2$. Hence, we have that
$$P^\prime(y_1)=P((y_1,y_2))=P^{\prime\prime}(y_2) \quad \text{and} \quad P^\prime(y_3)=P((y_3,y_4))=P^{\prime\prime}(y_4).$$

Let $\mathcal{P}^\prime$ and $\mathcal{P}^{\prime\prime}$ denote the marginal convex sets of probability distributions on $\mathcal{Y}_1$ and $\mathcal{Y}_2$, respectively, and let $\mathcal{P}$ denote the convex set of joint probability distributions on $\mathcal{Y}=\mathcal{Y}_1\times\mathcal{Y}_2$ \citep{couso}. Then, given our argument above, we have that $\text{Vol}(\mathcal{P}) < \text{Vol}(\mathcal{P}^\prime) + \text{Vol}(\mathcal{P}^{\prime\prime}) = 2 \text{Vol}(\mathcal{P})$. So in the general $|\mathcal{Y}|=|\mathcal{Y}_1|=|\mathcal{Y}_2|=2$ case where $y_1\neq y_3$ and $y_2\neq y_4$, the volume is subadditive.
\end{proof}

\begin{proof}[Proof of Proposition \ref{prop-2}]
Immediate from the assumption on the instance of SI.
\end{proof}

\begin{proof}[Proof of Theorem \ref{main_theorem}]
Pick any compact set $\mathcal{P}\subset\mathcal{M}(\Omega,\mathcal{F})$ and any set $\mathcal{Q}$ satisfying (a)-(c). Let $B^d_r \subset \mathbb{R}^d$ denote a generic ball in $\mathbb{R}^d$ of radius $r>0$. Notice that $N^{\text{pack}}_{r-\epsilon}(\mathcal{Q}^\prime)=N^{\text{pack}}_{r-\epsilon}(\mathcal{P})-N^{\text{pack}}_{r-\epsilon}(\mathcal{Q}) \geq 0$ because $\mathcal{P}\supset\mathcal{Q}$. Then, the proof goes as follows
    \begin{align}
        \frac{\text{Vol}(\mathcal{P})-\text{Vol}(\mathcal{Q}^\prime)}{\text{Vol}(\mathcal{P})} &= \frac{\frac{1}{c(r,d,\mathcal{P})}\text{Vol}(\tilde{\mathcal{P}}_r)-\frac{1}{c(r-\epsilon,d,\mathcal{Q}^\prime)}\text{Vol}(\tilde{\mathcal{Q}}^\prime_{r-\epsilon})}{\frac{1}{c(r,d,\mathcal{P})}\text{Vol}(\tilde{\mathcal{P}}_r)} \label{eq4}\\
        &\geq \frac{\text{Vol}(\tilde{\mathcal{P}}_r)-\text{Vol}(\tilde{\mathcal{Q}}^\prime_{r-\epsilon})}{\text{Vol}(\tilde{\mathcal{P}}_r)} \label{eq5}\\
        &=\frac{N^{\text{pack}}_r(\mathcal{P})\text{Vol}(B^d_r)-N^{\text{pack}}_{r-\epsilon}(\mathcal{Q}^\prime)\text{Vol}(B^d_{r-\epsilon})}{N^{\text{pack}}_r(\mathcal{P})\text{Vol}(B^d_r)} \label{eq6}\\
        &=\frac{N^{\text{pack}}_r(\mathcal{P})\text{Vol}(B^d_1)r^d-N^{\text{pack}}_{r-\epsilon}(\mathcal{Q}^\prime)\text{Vol}(B^d_1)(r-\epsilon)^d}{N^{\text{pack}}_r(\mathcal{P})\text{Vol}(B^d_1)r^d} \label{eq7}\\
        &=\frac{N^{\text{pack}}_r(\mathcal{P})r^d-N^{\text{pack}}_{r-\epsilon}(\mathcal{Q}^\prime)(r-\epsilon)^d}{N^{\text{pack}}_r(\mathcal{P})r^d} \nonumber\\
        &=1-\frac{N^{\text{pack}}_{r-\epsilon}(\mathcal{Q}^\prime)}{N^{\text{pack}}_r(\mathcal{P})}\left(1-\frac{\epsilon}{r} \right)^d \nonumber\\
        &\geq 1- \left( 1-\frac{\epsilon}{r} \right)^d, \label{eq8}
    \end{align}
    where \eqref{eq4} comes from equation \eqref{volume_eq}, \eqref{eq5} comes from the fact that $r-\epsilon \leq r \implies c(r-\epsilon,d,\mathcal{Q}^\prime) \geq c(r,d,\mathcal{P})$ by \eqref{decr_c}, \eqref{eq6} comes from $\tilde{\mathcal{P}}_r$ being the union of pairwise disjoint balls of radius $r$, \eqref{eq7} comes from properties of the volume of a ball of radius $r$ in $\mathbb{R}^d$, and \eqref{eq8} comes from property (c) of $\mathcal{Q}$.
\end{proof}

\section{High-dimensional probability}\label{hdp}
 Since Theorem \ref{main_theorem} in Section \ref{instab} is intimately related with Carl-Pajor's Theorem \citep{pajor}, we state (a version) of the theorem here. 
 
\begin{theorem}[Carl-Pajor]
Let $B_{1,d}$ denote the $d$-dimensional unit euclidean ball, and let $\mathcal{P} \subset B_{1,d}$ be a polytope with $m \in \mathbb{N}$ vertices. Then, we have
\begin{align}
  \frac{\text{Vol}(\mathcal{P})}{\text{Vol}(B_{1,d})} \leq \left(4 \sqrt{\frac{\log m}{d}}   \right)^d. 
\end{align}
\end{theorem}
For further results connecting high-dimensional probability and data science, see \cite{vershynin2018high}.

\end{document}